\documentclass{article}
\usepackage[utf8]{inputenc}
\usepackage{amssymb}
\usepackage{mathrsfs}
\usepackage{mathtools}  
\usepackage{amsthm} 
\usepackage{mathtools}
\usepackage{xcolor}
\usepackage{url}
\usepackage[ruled, lined, linesnumbered, commentsnumbered, longend]{algorithm2e}

\def\RR{\mathbb R}
\def\NN{\mathbb N}

\def\cA{\mathcal A}

\def\cC{\mathcal C}

\def\cG{\mathcal G}

\def\cJ{\mathcal J}

\def\cL{\mathcal L}
\def\cM{\mathcal M}
\def\cN{\mathcal N}
\def\cS{\mathcal S}

\def\cR{\mathcal R}

\def\ba{\mathbf a}

\def\bzero{\mathbf{0}}

\newcommand{\OPT}{\ensuremath{\textsc{Opt}}}

\newcommand{\raf}[1]{(\ref{#1})}

\newcommand*\Bell{\ensuremath{\boldsymbol\ell}}

\newtheorem{theorem}{Theorem}[section]

\newtheorem{lemma}[theorem]{Lemma}
\newtheorem{corollary}[theorem]{Corollary}

\newtheorem{definition}[theorem]{Definition}

\newtheorem*{remark*}{Remark}

\title{A Fully Polynomial Time Approximation Scheme for Constrained MDPs and Stochastic Shortest Path under Local Transitions}

\author{Majid Khonji\thanks{EECS Department,  Khalifa University, Abu Dhabi, UAE (email: \texttt{majid.khonji@ku.ac.ae}).}}
\begin{document}

\maketitle
\begin{abstract}
The fixed-horizon constrained Markov Decision Process (C-MDP) is a well-known model for planning in stochastic environments under operating constraints. Chance-Constrained MDP (CC-MDP) is a variant that allows bounding the probability of constraint violation, which is desired in many safety-critical applications. CC-MDP can also model a class of MDPs, called Stochastic Shortest Path (SSP), under dead-ends, where there is a trade-off between the probability-to-goal and cost-to-goal.
This work studies the structure of (C)C-MDP, particularly an important variant that involves local transition. In this variant, the state reachability exhibits a certain degree of locality and independence from the remaining states. More precisely, the number of states, at a given time, that share some reachable future states is always constant. (C)C-MDP under local transition is NP-Hard even for a planning horizon of two. In this work, we propose a fully polynomial-time approximation scheme for (C)C-MDP that computes (near) optimal deterministic policies. Such an algorithm is among the best approximation algorithm attainable in theory and gives insights into the approximability of constrained MDP and its variants.

\end{abstract}
\vspace{-10pt}
\section{Introduction}
The {\em Markov decision process} (MDP) \cite{howard1960dynamic} is a classical model for planning in uncertain environments. An MDP consists of states, actions, a stochastic transition function, a utility function, and an initial state. A solution of MDP is a policy that maps a state to an action that maximizes the global expected utility.  The {\em  stochastic shortest path} (SSP) \cite{bertsekas1991analysis}  is an MDP with non-negative utility values and involves a set of absorbing goal states. The problem has an interesting structure and can be formulated with a dual linear programming (LP) formulation \cite{d1963probabilistic} that can be interpreted as a minimum cost flow problem. Moreover,  MDPs  admit many heuristics-based algorithms \cite{bonet2003labeled,hansen2001lao} that utilize admissible heuristics to guide the search without exploring the whole state space.

Besides, constrained MDP (C-MDP) \cite{altman1999constrained} provides the means to add mission-critical requirements while optimizing the objective function. Each requirement is formulated as a  {\em budget constraint} imposed by a non-replenishable resource for which a bounded quantity is available during the entire plan execution. Resource consumption at each time step reduces the resource availability during subsequent time steps (see ~\cite{de2021constrained} for a detailed discussion).
A stochastic policy of C-MDP is attainable using several efficient algorithms (e.g., \cite{feinberg1996constrained}).  A heuristics-based search approach in the dual LP can further improve the running time for large state spaces \cite{trevizan2016heuristic}. For deterministic policies, however, it is known that C-MDP is NP-Hard for the finite-horizon case \cite{khonji2019approximability} (even when the planning horizon is only 2). The problem is also NP-Hard for the discounted infinite-horizon case~\cite{feinberg2000constrained}.

A special type of constraint occurs when we want to bound the probability of constraint violations by some threshold $\Delta$, which is often called a {\em chance constrained} MDP (CC-MDP). To simplify the problem,  \cite{dolgov2003approximating} proposes approximating the constraint using Markov's inequality, which converts the problem to C-MDP. Another approach by \cite{de2017bounding} applies Hoeffding's inequality on the sum of independent random variables to improve the bound. Both methods provide conservative policies that respect safety thresholds at the expense of the objective value (which could be arbitrarily worse than optimal).

In the partially observable setting, the problem is called  chance-constrained partially observable MDP (CC-POMDP). Several algorithms address CC-POMDP under risk constraints~\cite{santana2016rao,khonji2019approximability}. However, due to partial observability, these methods require an enumeration of histories, making the solution space exponentially large with respect to the planning horizon. To speed up the computation, \cite{sungkweon2021ccssp} provides an {\em anytime} algorithm using a Lagrangian relaxation method for CC-MDP and CC-POMDP that returns feasible sub-optimal solutions and gradually improves the solution's optimality when sufficient time is permitted. Unfortunately, the solution space is represented as an And-Or tree of all possible history trajectories, causing the algorithm to slow down as we increase the planning horizon.

The constrained MDP has a wide range of applications in AI and robotics.
One application of CC-MDP is navigation in a discretized environment (e.g., space exploration using a rover \cite{ono2015chance}). Some states (or grid coordinates) cause the agent to fail, say a cliff. The goal is to maximize utility (or science discovery) while avoiding dangerous states with a probability of $1-\Delta$. More applications for space landing and exploration are presented in \cite{ono2015chance}.
Another application of CC-MDP  is behavior planning for autonomous vehicles (AVs), which has been extensively studied in deterministic environments (see, e.g., \cite{wei2014behavioral}). One of the primary sources of uncertainty arises from drivers' intentions, i.e., potential maneuvers of agent vehicles \cite{huang2019online}. An effective behavior planner should optimize the maneuvers, say, minimize total commute time while bounding collision probability below some threshold. The objective of CC-MDP is to minimize the expected compute time, and the chance constraint is to bound the probability of collision. See \cite{khonji2021dual} for more empirical details.  
One application for C-MDP is a battery-operated unmanned aerial vehicle (UAV). The vehicle's goal is to maximize surveillance coverage, while the constraint is to keep energy consumption below battery capacity. See \cite{altman1999constrained} for a  list of constrained MDP applications\footnote{most of the C-MDP applications require non-negative parameters (costs and reward).}.
We refer to \cite{ding2020natural}, \cite{ahmadi2021constrained} for a more recent list of constrained MDP applications.


 In this work, we study a variant of (C)C-MDP in which the reachable set of states from a given  state intersects with at most a constant number  of reachable sets from any other states, denoted as (C)C-MDP under {\em local transitions}. This variant captures a class of MDP problems where state reachability exhibit a certain degree of {\em locality} such that only a constant number of states at a given time can share future states. 
The main contribution of this paper is a {\em fully polynomial time approximation scheme} (FPTAS) that computes (near) optimal deterministic policies for finite-horizon (C)C-MDP under local transitions in polynomial time. Since (C)C-MDP is shown to be NP-Hard (even under local transitions assumption \cite{khonji2019approximability}), our result is among the best possible approximation algorithm attainable in theory.

\vspace{-10pt}
\section{Problem Definition}
In this section, we provide a formal problem definition and relevant background.
\subsection{C-MDP.}
	A fixed-horizon constrained Markov decision process (C-MDP) is a tuple  \mbox{$ M= \langle \cS, \cA, T, U, s_0, h,C, P\rangle$}, 
		where $\cS$ and $\cA$ are finite sets of discrete {states} and {actions}, respectively;
		$T: \cS \times\cA\times \cS \rightarrow [0,1]$ is a probabilistic {transition function} between  states, 
		$T(s,a,s') = \Pr(s' \mid a,s)$, where $s,s' \in \cS$ and $a\in \cA$;
        $U: \cS \times \cA\rightarrow \RR_+$ is a non-negative {utility function};	
        $s_0$ is an initial state;
        $h$ is the planning horizon; 
        $C: \cS \times \cA\rightarrow \RR_+$ is a non-negative {cost function}; 	
        $P\in \RR_+$ is a positive upper bound on the cost.

A {\em deterministic} policy $\pi(\cdot,\cdot)$ is a function that maps a state and time step into action, $\pi:\cS\times \{0,1,...,h-1\} \rightarrow \cA$.  For simplicity, we write $\pi(s_k)$ to denote $\pi(s_k,k)$. A {\em run} is a sequence of random states  $S_0, S_1,\ldots, S_{h-1}, S_h$ that result from executing a policy, where $S_0 = s_0$ is known. 
The objective is to compute a policy that maximizes (resp. minimizes) the expected utility (resp. cost)  while satisfying  the constraint.
More formally,
\begin{align}
	\textsc{(C-MDP)}\quad&\quad\max_{\pi} \mathbb{E}\Big[\sum_{\mathclap{k=0}}^{h-1} U(S_{k}, \pi(S_{k}) )  \Big] \label{obj1}\\
	\text{Subject to }\quad &  \mathbb{E} \Big[\sum_{k=0}^{h-1} C(S_k, \pi(S_k))\mid \pi \Big] \le \ P.\notag
\end{align}

The MDP problem and its constrained variants can be visualized by a direct acyclic And-Or graph (DAG) $\cG$, where the vertices represent the states and  actions.
Thus, at depth $k$ the set of state nodes are the states that are reachable from previous actions at depth $k-1$, denoted as $\cS_{k}\subseteq \cS$. At each depth, we have at most $|\cS|$ states. Fig.~\ref{fig:MDP_graph} provides a pictorial illustration of the MDP And-Or (search) graph. Note that unlike And-Or search trees obtained by history enumeration algorithms (see, e.g., \cite{sungkweon2021ccssp}), with such representation, a node may have multiple parents, leading to a significant reduction in the search space. 
 \begin{figure}[ht]
     \centering 
     \vspace{-10pt}
     \includegraphics[scale=.7]{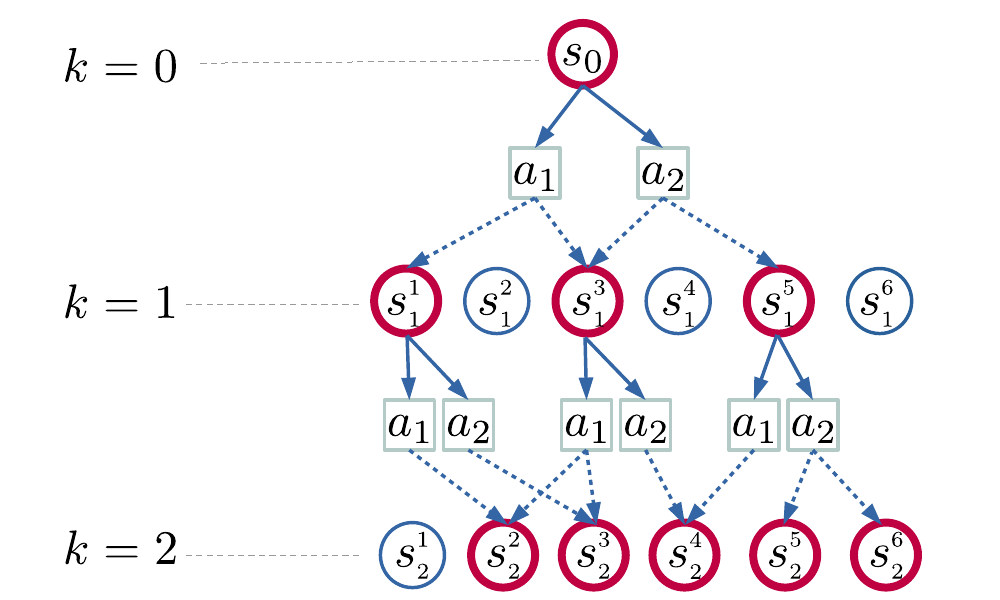}\vspace{-5pt}
     \caption{MDP graph where circles are state nodes and squares are actions. Circles with thick borders represent reachable states at time $k$, denoted as $\cS_k$.}\vspace{-10pt}
     \label{fig:MDP_graph}
 \end{figure}

The objective function and the constraint's left-hand side can be written recursively using the Bellman equation as 
\begin{align*}
v_\pi(s_k) &:= \sum_{\mathclap{s_{k+1} \in \cS_{k+1}}} T(s_k, \pi(s_k), s_{k+1}) v_\pi(s_{k+1}) + U(s_k, \pi(s_k)),\\  
c_\pi(s_k) &:= \sum_{\mathclap{s_{k+1} \in \cS_{k+1} }} T(s_k, \pi(s_k), s_{k+1}) c_\pi(s_{k+1}) + C(s_k, \pi(s_k)),
\end{align*}
for $k=0,...,h-1.$

\subsection{CC-MDP.}
A fixed-horizon chance-constrained MDP (CC-MDP) problem is formally defined  as a tuple 
$M=\langle \cS, \cA, T,U, s_0, h, r, \Delta \rangle,$ where $\cS, \cA, T, U, s_0, h$ are defined as in C-MDP, and
\begin{itemize}
	\item $r: \cS \rightarrow [0,1]$ is the probability of failure  at a given state;
	\item $\Delta$ is the corresponding risk budget, a threshold on the probability of failure over the planning horizon.
\end{itemize}
Let $R(s)$ be a Bernoulli random variable that indicates failure at state $s$, such that $R(s)=1$ if and only if $s$ is a risky state and zero otherwise. For simplicity, we write $R(s)$ to denote $R(s)=1$.
The objective of {CC-MDP} is to compute a deterministic policy (or a conditional plan) $\pi$ that maximizes (or minimizes) the cumulative expected utility (or cost) while bounding the probability of failure at {\em any} time step throughout the planning horizon. More precisely,
\begin{align}
	\textsc{(CC-MDP)}\quad&\quad\max_{\pi} \mathbb{E}\Big[\sum_{\mathclap{k=0}}^{h-1} U(S_{k}, \pi(S_{k}) )  \Big] \\
	\text{Subject to }\quad &  \Pr\Big(\bigvee_{k=0}^{h} R(S_k) \mid \pi \Big) \le \Delta. \label{con1}
\end{align}

To better understand Cons.~\raf{con1}, define the {\em execution risk} of a run at state $s_k$  as 
$$\textsc{Er}_{\pi}(s_k) :=  \Pr\Big(\bigvee_{k'=k}^{h} R(S_{k'}) \mid S_k = s_k\Big).$$ 
According to the definition, Cons.~\raf{con1} is equivalent to $\textsc{Er}_\pi(s_0) \le \Delta$. The lemma below shows that such constraint can be computed recursively. 

\begin{lemma}[\cite{khonji2021dual}]\label{lem:rec} 
	The execution risk of policy $\pi$ can be written as 
	\begin{equation*}
	\footnotesize
    \textsc{Er}_{\pi}(s_k)  = 
	    \left\{\begin{array}{l}
	     r(s_k)+(1-r(s_k)) \\
         \hspace{35pt}\displaystyle\sum_{\mathclap{s_{k+1}\in \cS}}\textsc{Er}_{\pi}(s_{k+1}) \pi(s_k)  T(s_k, a, s_{k+1}), \vspace{-10pt}\\
         
	     \hfill\text{if } k=0,...,h-1, \\
          r(s_h),\hfill \text{if } k=h.
	    \end{array}\right.
	\end{equation*}
\end{lemma}

CC-MDP captures a class of MDPs called stochastic shortest path (SSP), where there is a set of absorbing goal states and dead-end states. A measure of  policy feasibility under SSP is to have the probability-to-goal above some threshold $\epsilon$ \cite{freire2019exact}. This problem can be modeled as CC-MDP, where all non-goal states at horizon $h$ are considered risky states. Hence, the probability of failure is set to $\Delta = 1-\epsilon$. SSPs are often defined with infinite horizons. One can reduce SSP to fixed-horizon CC-MDP by successively incrementing the horizon $h$ until a feasible policy is attainable.
\subsection{Assumptions.}
 In this work, we study a variant of (C)C-MDP in which the number of state-action pairs that share subsequent states is bounded.  Such extension is denoted as (C)C-MDP under {\em local transition}. More formally, define the set of potential next states after executing action $a$ from state $s_k \in \cS_k$ for $k=0,...,h-1$ by,
 \begin{align*}
 	N_a(s_k) &:= \{s_{k+1} \mid T(s_k,a,s_{k+1}) > 0, s_{k+1}\in \cS_{k+1} \},\\
 	N_a(s_h) &:= \varnothing.
 \end{align*}Let $\textsc{Reach}(s_k)$ be a set of states in the MDP graph reachable from state $s_k$; more precisely,
 \begin{equation*}
\textsc{Reach}(s_k) := \left\{\begin{array}{l l}
\bigcup_{\substack{s_{k+1}\in N_a(s_k)\\ a \in \cA}} \textsc{Reach}(s_{k+1}) & \text{ if } k < h \\
s_k & \text{ if } k = h.
\end{array}\right. 
 \end{equation*}

\begin{definition}[Local Transition] \label{def:loc}
An MDP graph $\cG$  under local transition satisfies 
$$\big| \{s'_{k} \in \cG \mid \textsc{Reach}(s_k) \cap \textsc{Reach}(s'_{k}) \neq \varnothing\} \big| \le \psi,$$
for any $ s_k \in \cG$, where $\psi$ is a constant\footnote{We mean by a constant that the number is relatively small and is not a function of MDP instance $M$.}.
\end{definition}
 When $\psi = 0$, we call our problem (C)C-MDP under {\em disjoint transition}. The And-Or graph under the disjoint transition assumption is, in fact, an And-Or tree. Such structure helps to easily obtain a dynamic programming structure that is exploited in our algorithms, shown in the next subsections. 
 
 We also assume that all parameters (utility and cost values) are non-negative. Without such an assumption, one can show as in \cite{chau2016truthful} that the problem is inapproximable (i.e., no $\alpha$-approximation algorithm exists unless P=NP). 

\subsection{Benchmark.}
To analyze our algorithms, we rely on the notion of approximation algorithms. 
The subject of approximation algorithms is well-studied in the theoretical computer science community \cite{Vazirani10}. As follows, we define some standard terminology for approximation algorithms.
Consider a maximization problem $\Pi$ with non-negative objective function $f(\cdot)$; let $F$ be a feasible solution to  $\Pi$ and  $F^\star$ be an optimal solution to $\Pi$. $f(F) $ denotes the objective value of $F$. Let $\OPT= f(F^\star)$ be the optimal objective value of  $F^\star$. A common definition of approximate solutions is $\alpha$-approximation, where $\alpha$ characterizes the approximation ratio between the approximate solution and an optimal solution.

\begin{definition}[\cite{Vazirani10}]
	For $\alpha\in[0,1]$, an $\alpha$-approximation to maximization problem $\Pi$  is an algorithm that obtains a feasible solution $F$ for any instance such that $f(F)\ge \alpha \cdot \OPT.$
	\label{def:alpha}
\end{definition}

In particular, {\em fully polynomial-time approximation scheme} (FPTAS) is a $(1-\epsilon)$-approximation algorithm  to a maximization problem, for any $\epsilon>0$.  The running time of a FPTAS is polynomial in the input size and for every fixed $\tfrac 1 \epsilon$.  In other words, FPTAS allows to trade the approximation ratio against the running time.

In the following, we first study the problem under disjoint transition and then extend the result to the local transition case.
\vspace{-15pt}
\section{Algorithm}

For simplicity, we first study a special variant of (C)C-MDP in which actions stochastically lead to a small number of potential states, denoted as (C)C-MDP under {\em limited transition}. 
\begin{definition}[Limited Transition] \label{def:loc}
There exists a constant $\gamma\in \NN_+$ such that 	$N_a(s) \le \gamma$ for all  $a \in \cA, s \in \cS$.
\end{definition}

In the next subsection, we study (C)C-MDP under limited and disjoint transition, whereas in the following subsection, we relax the limited transition assumption. In the last subsection, we present an FPTAS for (C)C-MDP under local transition assumption, which generalizes the former cases.

\vspace{-5pt}
\subsection{FPTAS for (C)C-MDP under Limited and Disjoint Transition}
\label{sec:fptas-limited}

\begin{algorithm}[ht]	
	\footnotesize
	\caption{\texttt{lim-DynMDP}$[M, \epsilon]$}\label{alg}
	\SetKwInOut{Input}{Input}
	\SetKwInOut{Output}{Output}
	\Input{An instance of CC-MDP $M$; a parameter $\epsilon$ for the approximation guarantee}
	\Output{A deterministic policy $\pi$ }
	${\tt DP}_{\textsc{Er}}(s_k,\ell_k) \leftarrow \infty$; ${\tt DP}_{\sc \pi}(s_k,\ell_k) \leftarrow \varnothing$, ${\tt DP}_{\bar \Bell}(s_k,\ell_k) \leftarrow \bzero$,   for $k=0,...,h$;\\
    ${\tt DP}_{\textsc{Er}}(s_h,\ell_h) \leftarrow r(s_h),$ for all $s_h\in \cS_h, \ell_h  \in \cL_k = \{0\}$ \\
    
	\For{$k=h-1,...,0; s_k\in \cS_k; \ell_k \in \cL_k$ }{
		${\tt DP}_{\textsc{Er}}(s_k,\ell_k), {\tt DP}_{\pi}(s_k,\ell_k), {\tt DP}_{\bar \Bell}(s_k,\ell_k)  \leftarrow \texttt{Update}[s_k,\ell_k, \{\texttt{DP}_{\textsc{Er}}(s,\ell)\}_{s\in \cS_{k+1}, \ell \in \cL_{k+1}}] $ \label{alg:update}
	}
	$\pi \leftarrow \texttt{Fetch-Policy}[\texttt{DP}]$

	\Return $\pi$  

\end{algorithm}
\begin{algorithm}[ht] 
	\footnotesize
	\caption{\mbox{\texttt{Update}$\big[s_k,\ell_k, \{\texttt{DP}(s,\ell)\}_{s\in \cS_{k+1}}^{\ell\in \cL_{k+1}}  \big]$}}\label{alg:update-lim}
	$\textsc{Er}(s_k)\leftarrow \infty$\\
	$\textsc{Act} \leftarrow \varnothing$\\
    $\textsc{Alloc} \leftarrow \bzero$\\
	\For{$a \in \cA$}
	{
		
		// Find an allocation $\overline \Bell_{k+1}$ that acheives the minimum execution risk for action $a$ such that the total utility value is at least $\ell_k$ \\
		\For{$\overline \Bell_{k+1} = (\overline \ell^{1}_{k+1}, \overline \ell^{2}_{k+1},...) \in \cL_{k+1}^{|N_a(s_k)|}$ \label{alg:alc}}{
			$ \overline  v_a(s_k, \overline \Bell_{k+1}) \leftarrow \Big \lfloor \frac{1}{L_k} \Big( \sum_{s^i_{k+1} \in N_a(s_k)} T(s_k, a, s^i_{k+1}) \overline \ell^{i}_{k+1} + U(s_k, a) \Big) \Big\rfloor \cdot L_k$ \label{alg:bel1}\\
			\If{$\overline v_a(s_k, \overline \Bell_{k+1}) \ge \ell_k$}{
				$\textsc{Er}_a(s_k, \overline \Bell_{k+1}) := r(s_k)+(1-r(s_k)) \sum_{s^i_{k+1} \in N_a(s_k)} T(s_k, a, s^i_{k+1}) \cdot \texttt{DP}_{\textsc{Er}}(s^i_{k+1},\overline \Bell^i_{k+1})$ \label{alg:er0}\\
				\If{$\textsc{Er}_a(s_k, \overline \Bell_{k+1}) < \textsc{Er}(s_k)$}{ \label{alg:er1}
					$\textsc{Er}(s_k) \leftarrow \textsc{Er}_a(s_k, \overline \Bell_{k+1})$\\
					$\textsc{Act} \leftarrow a$ \label{alg:er2}\\
                    $\textsc{Alloc} \leftarrow \overline \Bell_{k+1}$
                    
				}
			}
		}
	}
	\Return $\textsc{Er}(s_k), \textsc{Act}, \textsc{Alloc}$
\end{algorithm}

\begin{algorithm}[ht]
	\footnotesize
	\caption{\texttt{Fetch-Policy}$[\texttt{DP}]$}\label{alg:policy}

     $\cC_k \leftarrow \varnothing$ for $k=1,...,h-1$ \\
	$\ell_{0} \leftarrow$ Find the maximum $\ell_{0} \in \cL_{0}$ such that ${\tt DP}_{\textsc{Er}}(s_{0},\ell_{0}) \le \Delta$ and  ${\tt DP}_{\textsc{Er}}(s_{0},\ell_{0} + L_{0}) > \Delta$  \label{alg:f1}\\
    $\cC_0 \leftarrow \{(s_0, \ell_0)\}$ \\
    \For{$k=0,...,h-1$}{
    	  \For{$(s^i_k, \overline \ell^i_k) \in \cC_k$}{
                      $\pi(s^i_{k}) \leftarrow {\tt DP}_{\pi}(s^i_{k},\overline\ell_{k}^i)$  \\
                      $\overline \Bell_{k+1}\leftarrow {\tt DP}_{\bar \Bell}(s^i_{k},\overline\ell^i_{k})$            \\
    					$\cC_{k+1} \leftarrow \cC_{k+1} \cup \{(s^i_{k+1}, \overline \ell^i_{k+1})\}_{s^i_{k+1} \in N_{a}(s^i_{k}) }$ where $a=\pi(s^i_{k})$\label{alg:f2} 
           }
	}
	\Return $\pi$
\end{algorithm}

The procedure involves constructing a 2-dimensional dynamic programming table, ${\tt DP}(\cdot,\cdot)$, where each cell ${\tt DP}(s_k,\ell_k)$ corresponds to state $s_k \in \cS_k$, and a discrete utility value $\ell_k$ (which  we will clarify next). Each cell contains three quantities,  ${\tt DP}_{\textsc{Er}}(s_k,\ell_k)\in \RR_+$ which maintains the minimum execution risk from state $s_k$, executing an action  that accrues a total value of at least $\ell_k$;  ${\tt DP}_\pi(s_k,\ell_k)=a$, the corresponding policy action $a$; and ${\tt DP}_{\bar \Bell}(s_k,\ell_k)\in \RR_+^{|N_a(s_k)|}$, a value allocation for subsequent states as we see next.
The main idea behind the algorithm lies in a utility discretization procedure that shrinks the set of possible values at a given state into a manageable number, exploiting the limited  and disjoint transition assumptions. A detailed description is provided in Algorithm {\tt lim-DynMDP} (Alg.~\ref{alg}). The algorithm relies on two subroutines, {\tt Update} (Alg.~\ref{alg:update-lim}) and {\tt Fetch-Policy} (Alg.~\ref{alg:policy}).  The former computes a discretized version of the Bellman equation along with the corresponding execution risk, and the latter recursively extracts the corresponding policy.  Line~\ref{alg:bel1} of {\tt Update} computes a discretized version of the Bellman equation under discretized future rewards, and Line~\ref{alg:er0} recursively computes the execution risk based on Lemma~\ref{lem:rec}. 
The pseudo-code is provided herein is for CC-MDP; however, it is also applicable to C-MDP with minor modifications. Namely, Line~\ref{alg:er0} of {\tt Update} should be replaced by
\begin{align*} 
\textsc{Er}_a(s_k, \overline \Bell_{k+1})\leftarrow& \sum_{\mathclap{s^i_{k+1} \in N_a(s_k)}} T(s_k, a, s^i_{k+1}) \\
&\quad \cdot \texttt{DP}_{\textsc{Er}}(s^i_{k+1}, \overline \Bell^i_{k+1}) + C(s_k, a)
\end{align*}
and $\Delta$  by $P$ in  Line~\ref{alg:f1} of {\tt Fetch-Policy}. 
Thus, all results in this paper apply to C-MDP as well. (In the remaining text, the term {\em execution risk} in the context of C-MDP would refer to the {\em total cost} instead.) 
Let $U_{\max}:=\max_{s\in \cS, a \in \cA} U(s,a)$ be the maximum utility of an action. 
Denote a discrete set of values $\cL_k$ for each time step $k=0,...,h$ as 
\begin{align} \label{eq:defL}
	\cL_k &:= \{0,L_k,2L_k,..., \lfloor \tfrac{U_{\max}\cdot (h-k)}{L_k}\rfloor L_k \}, \qquad\text{where } \notag \\
	 L_k &:=\frac{\epsilon U_{\max}}{(h-k)(\ln h + 1)}. 
\end{align}
Let $\pi$ be a solution returned by {\tt lim-DynMDP}, and $v_\pi(s_k):= \mathbb{E}[\sum_{k'=k}^{h-1} U(S_{k'}, \pi(S_{k'}) ) ]$ be the corresponding value function at state $s_k$. Similarly, denote $\pi^*$ to be an optimal solution, and $v_{\pi^*}(s_k)$  be the corresponding value function. Without loss of generality, assume that $v_{\pi^*}(s_0) \ge U_{\max}$
\footnote{If $U_{\max} = U(s_k,a) > v_{\pi^*}(s_0)$, then any policy that outputs action $a$ at state $s_k$ must be infeasible. Thus, such an action can be  deleted from the set of allowable actions at state $s_k$. Therefore, $U_{\max}$ can be taken as the second largest utility action and so on. The procedure can be performed in polynomial time as follows. Fix a policy $\pi(s_{k}) = a$, and  set the rest $\pi(s_{k'}')= a'_{k'}$ such that action $a'_{k'}$ achieves the minimum execution risk for $k'=h-1,...,0$ (computed recursively using Lemma~\ref{lem:rec}). If the solution is infeasible, repeat the procedure at different $k$. If again infeasible, one can safely drop $U(s_k,a)$, consider the next largest utility, and then repeat the procedure.}.
Define $\overline v_\pi(s_0)$ (resp.,  $\overline v_{\pi^*}(s_0)$) to be a discretized objective value computed recursively by,
\begin{equation}\footnotesize
	\overline v_\pi(s_k) = \bigg\lfloor \tfrac{1}{L_k}\Big(\sum_{\mathclap{s_{k+1} \in \cS_{k+1}}} T(s_k, \pi(s_k), s_{k+1}) \cdot \overline v_\pi(s_{k+1}) + U(s_k, \pi(s_k)\Big)\bigg\rfloor L_k. \label{eq:recd}
\end{equation}
The above equation corresponds to  step~\ref{alg:bel1} of {\tt Update}.

\begin{lemma} \label{lem:opt} 
	Let $\pi$ be a policy obtained by  {\tt lim-DynMDP} and $\pi^*$ be an optimal deterministic policy. The policy $\pi$ is feasible and satisfies 
	 $\overline v_{\pi}(s_0) \ge \overline v_{\pi^*}(s_0)$.
\end{lemma}
\begin{proof}
	We show (by induction) that for some $\ell_k \in \cL_k$ and $\overline \Bell_{k+1} \in \cL_{k+1}^{|N_a(s_k)|}$, there exists an action $a$ such that  
	$\ell_k = \overline v_a(s_k, \overline \Bell_{k+1}) \ge \overline v_{\pi^*}(s_k),$
	where $\overline v_a(\cdot,\cdot)$ is defined in Line~\ref{alg:bel1} of {\tt Update}. The algorithm enumerates all values of $\overline \Bell_{k+1}$ such that it attains the minimum execution risk for every $\ell_k \in \cL_k$. Throughout recursion, the procedure ensures that a feasible solution $\pi$ can be constructed such that $\ell_k = \overline v_a(s_k, \overline \Bell_{k+1}) \ge \overline v_{\pi^*}(s_k)$, as shown by {\tt Fetch-Policy}.
    
	We proceed with the induction proof; for the base case, $k=h-1$, we have
	$ \overline v_a(s_{h-1}, \overline \Bell_{h}) = \lfloor U(s_{h-1}, a)/L_{h-1} \rfloor L_{h-1}$. Clearly, there is an action that satisfies the claim.
    For the inductive step, suppose the claim holds at step $k$, we show that the claim also holds for step $k-1$. Note that algorithm  {\tt lim-DynMDP} enumerates all discretized allocations $\overline \Bell_k$ at step $k-1$  (as per Step~\ref{alg:alc} of {\tt Update}). Also note that $\overline v_{\pi^*}(s_k) \in \cL_k$ as the largest element in set $\cL_k$, defined in Eq.~\raf{eq:defL},  satisfies$\lfloor {U_{\max}(h-k)}/{L_k}\rfloor L_k \ge \overline v_{\pi^*}(s_k)$. Hence, there exists an allocation $\overline \Bell_k$ such that each $i$-th element $\overline \ell_k^i = \overline v_{a}(s^i_k, \overline \Bell_{k+1}) \ge \overline v_{\pi^*}(s^i_k)$ for some $\overline \Bell_{k+1}$ (inductive assumption).
	Hence, there exists an $\ell_{k-1}$ and an action $a$ such that
	\begin{align}
		 &\ell_{k-1}= \overline v_a(s_{k-1}, \overline \Bell_k) \notag\\
		 &= \bigg\lfloor \tfrac{1}{L_{k-1}}\Big(\sum_{\mathclap{s^i_k \in N_a(s_{k-1})}} T(s_{k-1}, a, s^i_k) \overline \ell^{i}_{k} + U(s_{k-1}, a) \Big)\bigg\rfloor L_{k-1}\notag \\
		&\ge  \bigg\lfloor \tfrac{1}{L_{k-1}}\Big(\sum_{\mathclap{s_{k} \in \cS_k}} T(s_{k-1}, a, s_{k}) \overline v_{\pi^*}(s_k) + U(s_{k-1}, a)\Big)\bigg\rfloor L_{k-1} \notag \\
		&= \overline v_{\pi^*}(s_{k-1}),
	\end{align}
	 where the inequality follows by the inductive assumption. 

It remains to show that such action $a$ is feasible.
	 Each cell ${\tt DP}_{\pi}(\cdot, \cdot)$ corresponds to an action that achieves the minimum execution risk that accrues a total value of at least $\ell_k$. Since the execution risk is a non-decreasing function (Line~\ref{alg:er0} of {\tt Update}), ${\tt DP}_{\textsc{Er} }(s_k, \ell_{k}) \le  {\textsc{Er}}_{\pi^*}(s_{k}) \le  \Delta$ for some $\ell_k =  \overline v_{\pi^*}(s_k)$. By the disjoint transition assumption, there is a unique state $s_{k-1}$  that involves the row $\texttt{DP}_{\textsc{ER}}(s_k, \cdot)$, $s_k \in N_a(s_{k-1})$, in computing $\texttt{DP}_{\textsc{ER}}(s_{k-1},  \ell_{k-1})$ (Line~\ref{alg:er0} of {\tt Update}). Hence, only table cells related to state $s_{k-1}$ sets the values of $\overline\ell_k^i$ of the subsequent states $s_k \in N_a(s_{k-1})$. As each cell $\texttt{DP}_{\pi}(s_{k},  \overline \ell^i_{k})$ corresponds to a single action, no two $s_{k-1}, s'_{k-1}$ share subsequent state $s_k$, and only one cell among row $\texttt{DP}_{\pi}(s_{k-1},\cdot)$ is backtracked by {\tt Fetch-Policy},  the policy remains consistent, i.e., it outputs a single action for each state. (Note that this is not the case if $s_k$ has multiple parents in the And-Or graph, which is the case under local transition assumption.) Such  c                                                                                                                                                                                                                                                                                                                                                                                                                                                                                                         action is backtracked by {\tt Fetch-Policy}. Therefore, policy $\pi$ is feasible. 
\end{proof}

\begin{lemma}\label{lem:dis}
	An optimal deterministic policy $\pi^*$ satisfies    $\overline v_{\pi^*}(s_{k}) \ge v_{\pi^*}(s_k)- \sum_{k'=k}^{h-1} L_{k'}.$
\end{lemma}
\begin{proof}
	We proceed with an inductive proof. For the base case, computing  Eq.~\raf{eq:recd} for $\pi^*$ at $k=h-1$, we have
	\begin{align}
		&\overline v_{\pi^*}(s_{h-1})  = \Big\lfloor \tfrac{U(s_{h-1}, \pi^*(s_{h-1}))}{L_{h-1}}\Big\rfloor \cdot L_{h-1}\notag\\
		& \ge U(s_{h-1}, \pi^*(s_{h-1})) - L_{h-1} =   v_{\pi^*}(s_{h-1}) - L_{h-1},
	\end{align}
    which follows using the property $\lfloor \frac{x}{y}\rfloor y \ge x - y$ for $x,y \in \RR_+$.
	For the inductive step, suppose that we have,
	$	\overline v_{\pi^*}(s_{k}) \ge v_{\pi^*}(s_k)- \sum_{k'=k}^{h-1} L_{k'}.$
	We compute the corresponding inequality for $s_{k-1}$ as follows,
	\begin{align}
		&\overline v_{\pi^*}(s_{k-1})  = \bigg\lfloor \tfrac{1}{L_{k-1}}\Big(\sum_{s_{k} \in \cS_k} T(s_{k-1}, \pi^*(s_{k-1}), s_{k}) \cdot \overline v_{\pi^*}(s_{k}) \notag \\
        &\qquad + U(s_{k-1}, \pi^*(s_{k-1}))\Big)\bigg\rfloor L_{k-1} \\
		& \ge \bigg\lfloor \tfrac{1}{L_{k-1}}\Big(\sum_{s_{k} \in \cS}  T(s_{k-1}, \pi^*(s_{k-1}), s_{k}) \cdot \big( v_{\pi^*}(s_{k}) \notag \\
        &\qquad -  \sum_{k'=k}^{h-1}     L_{k'} \big) + U(s_{k-1}, \pi^*(s_{k-1}))\Big)\bigg\rfloor L_{k-1}, \label{eq:inductive}
   \end{align}
   where Eq.~\raf{eq:inductive} follows by the inductive assumption. Since $T(\cdot,\cdot,\cdot)$ is a probability function that adds up to one, and using the property $\lfloor \frac{x}{y}\rfloor y \ge x - y$ for $x,y \in \RR_+$, we obtain,
   \begin{align}
		&\overline v_{\pi^*}(s_{k-1})\ge  \sum_{s_{k} \in \cS}  T(s_{k-1}, \pi^*(s_{k-1}), s_{k}) \cdot  v_{\pi^*}(s_{k}) \notag \\
        &\qquad + U(s_{k-1}, \pi^*(s_{k-1})) - \sum_{k'=k-1}^{h-1}   L_{k'} \notag\\
		&= v_{\pi^*}(s_{k-1}) - \sum_{k'=k-1}^{h-1}   L_{k'}, \label{eq:lem2f}
	\end{align}
	which completes the inductive proof. 
\end{proof}

\begin{corollary} \label{thm:fptas1}
	{\tt lim-DynMDP} is an FPTAS for (C)C-MDP under limited and disjoint transition assumptions.
\end{corollary}
\begin{proof}
First, observe that the algorithm maintains a dynamic programming table with minimum execution risk. Subroutine {\tt Fetch-Policy} ensures that a feasible solution with such property is retrieved. The algorithm  runs in $O\big(\big(\tfrac{h^{2} \ln h+1}{\epsilon}\big)^{\gamma + 1} |\cA| |\cS|\big)$.  Note that by the limited transition assumption, $\gamma$ is a constant; thus, the running time is polynomial.
By Lemma~\ref{lem:dis} and by the definition of $L_k$ given in Eq.~\raf{eq:defL}, we obtain
	\begin{align}
		\overline v_{\pi^*}(s_0)  &\ge v_{\pi^*}(s_0) - \sum_{k=0}^{h-1}   L_{k}  \notag \\
        &= v_{\pi^*}(s_0) - \sum_{k=0}^{h-1} \frac{\epsilon U_{\max}}{(h-k)(\ln h + 1)} \notag \\
		&= v_{\pi^*}(s_0) - \frac{\epsilon U_{\max}}{\ln h + 1} \sum_{k=0}^{h-1} \frac{1}{(h-k)}\notag \\
		&\ge v_{\pi^*}(s_0) - \frac{\epsilon U_{\max}}{\ln h + 1} (\ln h + 1)  \label{eq:har}\\
		&\ge (1-\epsilon)\cdot v_{\pi^*}(s_0), \label{eq:har2}
	\end{align}
where Eq.~\raf{eq:har} follows by using an upper bound on the harmonic series, $\sum_{n=1}^k \tfrac{1}{n} \le \ln k + 1$. 
By Lemma~\ref{lem:opt} and Eq.~\raf{eq:har2},
	$v_\pi(s_0) \ge  \overline  v_{\pi}(s_0) \ge   \overline v_{\pi^*}(s_0) \ge (1-\epsilon) v_{\pi^*}(s_0),$
which completes the proof.
\end{proof}

\vspace{-5pt}
\subsection{FPTAS for (C)C-MDP under Disjoint Transition}\label{sec:fptas-disjoint}

In this section, we relax the limited transition assumption and show how to obtain an FPTAS for (C)C-MDP. In other words, we assume $\gamma$ is a polynomial in definition~\ref{def:loc}.
The main idea behind our algorithm is to improve {\tt Update} subroutine to avoid  full enumeration of $\overline \Bell_{k+1}$, which is exponential in the number of subsequent states $|N_a(s_k)|$. Such enumeration could be feasible under the limited transition assumption, but not in general. We show here how the structure of this step could be exploited. Notably, finding an allocation that achieves the minimum execution risk such that the total utility value is at least $\ell_k$ is a slight generalization for a well-known problem called {\em minimum Knapsack} (MinKS) \cite{pruhs2007approximation,bentz2016note}. 
More formally, 
\begin{definition} \label{def:ks}
	Multiple-choice minimum Knapsack problem (McMinKS) is defined as follows. Given a set of categories $\cN$, and a set of allowable choices $\cM_i$ per category $i\in \cN$, an item $(i,j)$ is defined by weight $w_{i,j} \in \RR_+$ and value $v_{i,j}\in \RR_+$ for $i \in \cN$ and $j \in \cM_i$. The goal is to select one item  from the allowable choices  $\cM_i$ per category $i$ (hence the name multiple-choice) such that the total weight is minimized, and the total value is at least $D \in \RR_+$.
\end{definition}
The problem can be formally defined as an {\em integer linear program} (ILP) as follows.
\begin{align}
\text{(McMinKS)}\quad & \min_{x_{i,j} \in \{0,1\}}  \sum_{i\in \cN} \sum_{j \in \cM_i} w_{i,j} x_{i,j}, \notag\\
\text{Subject to}\quad & \sum_{i \in \cN} \sum_{j\in \cM_i} v_{i,j} x_{i,j} \ge D, \label{con}\\
& \sum_{j \in \cM_i} x_{i,j} = 1, \quad \text{for all } i \in \cN.
\end{align}

Algorithm \ref{alg:update2}, denoted by {\tt KS-Update}, presents a reduction from the allocation subproblem to McMinKS in Lines~\ref{alg:ks0}-\ref{alg:ksl}. Indeed, finding an optimal solution for the corresponding McMinKS instance will obtain an FPTAS (Lemma~\ref{lem:opt} holds and hence Corollary~\ref{thm:fptas1} proof follows). However, MinKS is NP-Hard \cite{pruhs2007approximation,kellerer2004knapsack}, therefore our best bet is to find an approximate solution in polynomial time.
Although there is an FPTAS for MinKS, the approximation guarantee is provided on the objective function, which in our case, following the reduction, is the constraint for  the original (C)C-MDP problem. Thus, we need an algorithm that bounds the constraint violation of McMinKS (which is the objective of (C)C-MDP, following the reduction above).
some modifications are needed to the algorithm to obtain a bounded McMinKS constraint violation and handle the multiple-choice extension (as we will see next).

\begin{algorithm}[ht] 
	\footnotesize
	\caption{\mbox{\texttt{KS-Update}$[s_k,\ell_k, \{\texttt{DP}(s,\ell)\}_{s\in \cS_{k+1}}^{\ell\in \cL_{k+1}}  ]$}}\label{alg:update2}
	$\textsc{Er}(s_k)\leftarrow \infty$; $\textsc{Act} \leftarrow \varnothing$;   $\textsc{Alloc} \leftarrow \bzero$ \\
	\For{$a \in \cA$}
	{
		// Find an allocation $\overline \Bell_{k+1}$ that achieves the minimum execution risk for action $a$ such that the total utility value is at least $\ell_k$ \\
		Let $\cN := \{1,...,|N_a(s_k)|\}$ \label{alg:ks0}\\
		Let $\cM_i := \cL_{k+1}$ for $i\in \cN$ \\
		Let $w_{i,j} := T(s_k, a, s^i_{k+1}) \cdot \texttt{DP}_{\textsc{Er}}(s^i_{k+1}, \overline \ell^i_{k+1})$ for all  $j=\overline \ell^i_{k+1} \in \cL_{k+1}$ and  $s^i_{k+1} \in N_a(s_k)$ \label{alg:redw}\\
		Let $v_{i,j} := {T(s_k, a, s^i_{k+1}) \cdot \overline \ell^i_{k+1}}$ for  $j=\overline\ell^i_{k+1} \in \cM_i$ and $i$ such that  $s^i_{k+1} \in N_a(s_k)$\\
		Let $D:= \ell_k - U(s_k, a)$ \label{alg:ksl}\\
		$\overline \Bell_{k+1} \leftarrow$  \texttt{Dyn-MinKS}$[(w_{i,j}, v_{i,j})_{i\in \cN,j \in  \cM_i}, D]$ \label{alg:lineks}
		
		$\textsc{Er}_a(s_k, \overline \Bell_{k+1}) := r(s_k)+(1-r(s_k)) \sum_{s^i_{k+1} \in N_a(s_k)} T(s_k, a, s^i_{k+1}) \cdot \texttt{DP}_{\textsc{Er}}(s^i_{k+1}, \overline \ell^i_{k+1})$ \\
		\If{$\textsc{Er}_a(s_k, \overline \Bell_{k+1}) < \textsc{Er}(s_k)$}{ 
			$\textsc{Er}(s_k) \leftarrow \textsc{Er}_a(s_k, \overline \Bell_{k+1})$\\
			$\textsc{Act} \leftarrow a$\\
            $\textsc{Alloc} \leftarrow \overline \Bell_{k+1}$
		}
	
	}
	\Return $\textsc{Er}(s_k), \textsc{Act}, \textsc{Alloc}$
\end{algorithm}
\begin{algorithm}[ht] 
\footnotesize
	\caption{\texttt{Dyn-MinKS}$[(w_{i,j}, v_{i,j})_{i\in \cN,j \in  \cM_i}, D]$}\label{alg:ks}
	$\Bell=(\ell^i)_{i\in \cN} \leftarrow \bzero$ \\ 
	$\texttt{TB}(i, \rho) \leftarrow \infty$ for all  $i\in \cN$ and $\rho \in \cR$ \\
	$\texttt{TB}(0, \rho) \leftarrow 0$ for all $\rho$\\
	Let $\overline v_{i,j} := \Big\lfloor \frac{v_{i,j}}{R_k}\Big\rfloor\cdot R_k$ for all  $i\in \cN$ and  $j \in \cM_i$ \label{alg:discdem}\\
	\For{$i  = 1,..., |\cN|$}
	{
		\For{$\rho \in \cR_k = \Big\{0, R_k, 2R_k,..., \Big\lfloor \frac{D + \max_{i,j} v_{i,j} }{R_k} \Big\rfloor R_k \Big\}$} {
			$\texttt{TB}(i, \rho) \leftarrow \min_{j \in \cM_i} \texttt{TB}(i-1, [\rho - \overline v_{i,j}]^+) + w_{i,j} $\\
			$\texttt{Alc}(i,\rho) \leftarrow \arg\min_{j \in \cM_i} \texttt{TB}(i-1, [\rho - \overline v_{i,j}]^+) + w_{i,j}$,  where $[x]^+ = x$ if $x\ge 0$ and $[x]^+=0$ otherwise, for any $x \in \RR$
			
		}
	}
		Find minimum $\rho'$ such that $\rho' \ge D$\label{alg:fetch1}\\
		\For{$i = |\cN|, |\cN|-1,...,  1$ }{ 
			$j = \texttt{Alc}(i,\rho' )$;
			$\rho' \leftarrow \rho' -  \overline v_{i,j}$ \label{alg:fetch2};
            $\ell^i \leftarrow j$ 
		}
	\Return $\Bell$ 	
\end{algorithm}

Algorithm~\ref{alg:ks},  denoted as {\tt Dyn-MinKS},  gives a dynamic programming procedure to solve McMinKS within a bounded constraint violation. The algorithm rounds the values into a discrete set of possible values that provably can have a bounded constraint violation (as per Lemma~\ref{lem:minks} below).  
The set of possible {\em discretized} values $\cR_k$ is defined as 
\begin{equation}
	\cR_k := \Big\{0, 1R_k,..., \Big\lfloor \frac{D+ \max_{i,j} v_{i,j}}{R_k} \Big\rfloor R_k \Big\}, \label{eq:defR}
\end{equation}
where $R_k$ is a discretization factor defined below.
Let $\Bell$  be an allocation returned by algorithm  {\tt Dyn-MinKS} and $\Bell^*$ be an optimal solution. 
\begin{lemma}\label{lem:minks}
Algorithm {\tt Dyn-MinKS} obtains a solution $\Bell = (\ell^1,..., \ell^{|\cN|})$ that  satisfies
\begin{align*}
\sum_{i \in \cN}  w_{i,\ell^{i}} \le \sum_{i \in \cN}  w_{i,\ell^{i*}}, \text{ and }
\sum_{i \in \cN}  v_{i,\ell^{i}} \ge \sum_{i \in \cN}  v_{i,\ell^{i*}} -  |\cN| R_k,
\end{align*}
where $\Bell^*$ is an optimal solution.
\end{lemma}
\begin{proof}
Define $\overline v_{i, \ell^i} := \lfloor v_{i, \ell^i}/R_k \rfloor R_k$ as in Line~\ref{alg:discdem} of {\tt Dyn-MinKS} (also define $\overline v_{i,\ell^{i*}} := \lfloor v_{i, \ell^{i*}}/R_k \rfloor R_k$).   The algorithm maintains a table {\tt TB}$(i,\rho)$ of minimum total item weights up to category $i$ that satisfies a total value of at least $\rho$. 
Since the algorithm discretizes values (Line~\ref{alg:discdem}), and  the largest element of $\cR$ is an upper bound on $\sum_{i \in \cN}  \overline v_{i,\ell^{i*}} $  (by the definition of $\cR$ in Eq.~\raf{eq:defR}), then any discretized optimal total values are considered in the table.  Therefore in steps~\ref{alg:fetch1}-\ref{alg:fetch2}, the algorithm obtains a minimum $\rho \ge D$ that accrues the least total weight, hence $\sum_{i \in \cN}  w_{i,\ell^{i}} \le \sum_{i \in \cN}  w_{i,\ell^{i*}}$. By the feasibility of optimal solutions $\sum_{i \in \cN} \overline v_{i,\ell^{i*}} \ge D$, and since $\rho$ is the least element that satisfies $\rho \ge D$,  we have
\begin{equation} \label{eq:minks-d}
\rho = \sum_{i \in \cN} \overline v_{i,\ell^{i}} \ge \sum_{i \in \cN} \overline v_{i,\ell^{i*}}.
\end{equation}
Thus, by Eq.~\raf{eq:minks-d} and rounding values down (Line~\ref{alg:discdem} of {\tt Dyn-MinKS}), 
$\sum_{i \in \cN}  v_{i,\ell^{i}} \ge \sum_{i \in \cN} \overline v_{i,\ell^{i}} \ge \sum_{i \in \cN} \overline v_{i,\ell^{i*}} \ge \sum_{i \in \cN} (v_{i,\ell^{i*}} - R_k) = \sum_{i \in \cN} v_{i,\ell^{i*}} -  |\cN| R_k ,$
which completes the proof.

\end{proof}

We define algorithm {\tt dis-DynMDP} by replacing {\tt Update} at Line~\ref{alg:update}  of {\tt lim-DynMDP} by {\tt KS-Update}, and using the following discretization factors, 
\begin{align} \label{eq:newR}
L_k = \frac{\epsilon U_{\max}}{3(h-k) (\ln h + 1)} \quad \text{and}\quad R_k = \frac{L_k}{\gamma}.
\end{align}

\begin{lemma}\label{lem:opt2}
	Let $\pi$ be a policy obtained by  {\tt dis-DynMDP} and $\pi^*$ be an optimal deterministic policy. The solution $\pi$ satisfies 
	 $\overline v_{\pi}(s_k) \ge \overline v_{\pi^*}(s_k) - 2 \sum_{k'=k}^{h-1}\cdot L_{k'}.$ 
\end{lemma}
\begin{proof}
	We show (by induction) that for some $\ell_k \in \cL_k$ and action $a$, we have  
	$\ell_k = \overline v_\pi(s_k) \ge \overline v_{\pi^*}(s_k)- 2 \sum_{k'=k}^{h-1}\cdot L_{k'}.$
	
	We proceed with the induction proof; for the base case, $k=h-1$, we have
	$ \overline v_\pi(s_{h-1}) = \lfloor U(s_{h-1}, a)/L_{h-1} \rfloor L_{h-1}$. Clearly, there is an action that satisfies the claim.
    For the inductive step, suppose the claim holds at step $k$; we show that the claim also holds for step $k-1$. 
    
    Since the {\tt dis-DynMDP} considers all possible values for $\ell_{k-1}$ at time $k-1$, there exists an $\ell_{k-1}$, an action $a$, and a solution $\overline \Bell_k$ (Line \ref{alg:lineks} of  {\tt KS-Update}) such that, 
	{\footnotesize
    \begin{align}
		&\ell_{k-1} = \bigg\lfloor \tfrac{1}{L_{k-1}}\Big(\sum_{s^i_k \in N_a(s_{k-1})} T(s_{k-1}, a, s^i_k) \overline \ell^{i}_{k} \notag \\
        &\qquad \quad + U(s_{k-1}, a) \Big)\bigg\rfloor L_{k-1}\notag \\
        	&\ge  \bigg\lfloor \tfrac{1}{L_{k-1}}\Big(\sum_{s^i_{k} \in N_a(s_{k-1})} T(s_{k-1}, a, s^i_{k}) \overline \ell^{i*}_k -  |N_{a}(s_{k-1})| R_k \notag \\
            &\qquad \quad + U(s_{k-1}, a)\Big)\bigg\rfloor L_{k-1} \label{eq:bylemminks}
     \end{align}       
     }%
     where Eq.~\raf{eq:bylemminks} follows by Lemma~\ref{lem:minks} (where $v_{i,j}:= T(s_{k-1},a, s^i_k)\cdot \overline v_\pi(s^i_k)$ and $w_{i,j}:=T(s_{k-1}, a, s^i_k) \cdot \texttt{DP}_{\textsc{Er}}(s^i_k, \overline \ell^i_{k})$ as per Line~\ref{alg:redw} of {\tt KS-Update}). By the inductive assumption and Eq.~\raf{eq:bylemminks},
     {\footnotesize
     \begin{align*}
  		\ell_{k-1} &\ge  \bigg\lfloor \tfrac{1}{L_{k-1}}\Big(\sum_{s_{k} \in \cS} T(s_{k-1}, a, s_{k}) \big(\overline v_{\pi^*}(s_k) - 2\sum_{k'=k}^{h-1}L_{k'}\big) \notag \\
        &\qquad \quad - |\cN_{a}(s_{k-1})| R_{k-1} + U(s_{k-1}, a)\Big)\bigg\rfloor L_{k-1} 
  \end{align*}
  }%
Therefore,
  {\footnotesize
  \begin{align}
       &\ell_{k-1} \ge  \bigg\lfloor \tfrac{1}{L_{k-1}}\Big(\sum_{s_{k} \in \cS_k} T(s_{k-1}, a, s_{k}) \overline v_{\pi^*}(s_k) + U(s_{k-1}, a) \notag \\
       &\qquad\quad - 2\sum_{k'=k}^{h-1}L_{k'} - \gamma R_{k-1}\Big)\bigg\rfloor L_{k-1} \label{eq:dround0}
 \end{align}
 }%
 Thus, by the definition of $R_{k}$ in Eq.~\raf{eq:newR}, the r.h.s of Eq.~\raf{eq:dround0} can be written as,
 \begin{align}
        &\ell_{k-1}\ge  \bigg\lfloor \tfrac{1}{L_{k-1}}\Big(\sum_{s_{k} \in \cS_k} T(s_{k-1}, a, s_{k}) \overline v_{\pi^*}(s_k) + U(s_{k-1}, a) \notag \\
        &\qquad \quad -2\sum_{k'=k}^{h-1}L_{k'} - L_{k-1} \Big)\bigg\rfloor L_{k-1}  \label{eq:dround1}\\
       &\ge \overline v_{\pi^*}(s_k) -  2 \sum_{k'=k-1}^{h-1} L_{k'}  \label{eq:dround2},
	\end{align}
 By the disjoint transition assumption (following the feasibility argument in the proof of Lemma~\ref{lem:opt}), policy $\pi$ is feasible. 
\end{proof}

\begin{corollary} \label{thm:fptas2}
	Algorithm {\tt dis-DynMDP} is an FPTAS for (C)C-MDP under disjoint transition assumption.
\end{corollary}
\begin{proof}
First, observe that the algorithm maintains a dynamic programming table with minimum execution risk. Lines~\ref{alg:f1} of subroutine {\tt Fetch-Policy} ensures that a feasible solution with such property is constructed. The algorithm  runs in polynomial time as the sizes of $\cL_k$ and $\cR_k$ are polynomial. 
By Lemma~\ref{lem:dis} and Lemma~\ref{lem:opt2}, expanding for $s_0$, and by the definition of $L_k$ and $R_k$, we obtain
	\begin{align}
		\overline v_\pi(s_0)  &\ge \overline v_{\pi^*}(s_0)   - 2 \sum_{k=0}^{h-1} L_{k} \\
        &\ge  v_{\pi^*}(s_0) - \sum_{k=0}^{h-1}L_{k}  - 2 \sum_{k=0}^{h-1} L_{k}
    \end{align}
    Substituting $L_{k}$ obtains,
    \begin{align}
        & = v_{\pi^*}(s_0) - 3 \sum_{k=0}^{h-1} \frac{\epsilon U_{\max}}{3(h-k)(\ln h + 1)} \\
		&= v_{\pi^*}(s_0) - \frac{\epsilon U_{\max}}{\ln h + 1} \sum_{k=0}^{h-1} \frac{1}{(h-k)} 
   \end{align}
   Using the upper bound on the harmonic series $\sum_{n=1}^k \tfrac{1}{n} \le \ln k + 1$ obtains
   \begin{align}
        \overline v_\pi(s_0) &\ge v_{\pi^*}(s_0) - \frac{\epsilon U_{\max}}{\ln h + 1} (\ln h + 1)  \label{eq:harr2}\\
		&\ge (1-\epsilon)\cdot v_{\pi^*}(s_0), \
	\end{align}
 Therefore,
$v_\pi(s_0) \ge  \overline  v_{\pi}(s_0) \ge  (1-\epsilon) v_{\pi^*}(s_0),$
which completes the proof.

\end{proof}
\subsection{FPTAS for (C)C-MDP under Local Transition}
\label{sec:fptas}

\begin{algorithm}[ht]	
	\footnotesize
	\caption{\texttt{DynMDP}$[M, \epsilon]$}\label{alg2}
    ${\tt DP}_{\textsc{Er}}(s_k,\ell_k) \leftarrow \infty$; ${\tt DP}_{\sc \pi}(s_k,\ell_k) \leftarrow \varnothing$,  ${\tt DP}_{\bar \Bell}(s_k,\ell_k) \leftarrow \bzero$, for all $k=0,...,h-1$;  $s_k\in \cS_k$; $\ell_k \in \cL_k$ \\
    ${\tt DP}_{\textsc{Er}}(s_h,0) \leftarrow  r(s_h)$  \\ 
	\For{$k=h-1,...,0$ }{
    	$\cJ_k \leftarrow$ Partition states in $\cS_k$ based on reachability given by Eqn.~\raf{eq:reach}
        
        \For{$J_k\in \cJ_k$}{
        	\For{$\Bell_k=(\ell_k^c)_{s^c_k \in J_k} \in \cL_k^{|J_k|}$}{
		$\big({\tt DP}_{\textsc{Er}}(s^c_k,\ell^c_k), {\tt DP}_{\pi}(s^c_k,\ell^c_k), {\tt DP}_{\bar \Bell}(s^c_k,\ell^c_k) \big)\leftarrow \texttt{mKS-Update}[J_k,\Bell_k, \{\texttt{DP}_{\textsc{Er}}(\cdot,\cdot)\}] $

        	}
        }
	}
	$\pi \leftarrow \texttt{Fetch-Policy}[\texttt{DP}]$

	\Return $\pi$  

\end{algorithm}

\begin{algorithm}[ht] 
	\footnotesize
	\caption{\mbox{\texttt{mKS-Update}\small $\big[J_k,\Bell_k=(\ell_k^c)_{s^c_k \in J_k}, \{\texttt{DP}_{\textsc{Er}}(s,\ell)\}  \big
]$}}\label{alg:update3}
	$\textsc{Er}(s^c_k)\leftarrow \infty$ for $s_k^c \in J_k$\\
	$\textsc{Act} \leftarrow \varnothing$\\
    $\textsc{Alloc} \leftarrow \bzero$ \\
	 \For{$\ba=(a^c)_{s^c_{k} \in J_k} \in \cA^{|J_k|}$}{
		// Find an allocation  that achieves  the minimum execution risk of states in  $J_k$ such that the total utility at each  $s_k^c\in J_k$ is at least $\ell^c_k$\\
         Let $\cN := \bigcup_{s^c_k \in J_k} N_{a^c}(s^c_k) $ \label{alg:mks0}\\
         Let $\cM_i := \cL_{k+1}$ for $i\in \cN$ \\
         Let $w_{i,j} :=  \sum_{s_k^c \in J_k}T(s^c_k, a^c, s^i_{k+1}) \cdot \texttt{DP}_{\textsc{Er}}(s^i_{k+1}, \overline \ell^i_{k+1})$ for all  $j=\overline \ell^i_{k+1} \in \cM_i$ and  $i$ such that $s^i_{k+1} \in N_a(s^c_k)$\\
        \For{$ s^c_k \in J_k$}{
                
                Let $v^c_{i,j} := {T(s^c_k, a^c, s^i_{k+1}) \cdot \overline \ell^i_{k+1}}$ for all  $j=\overline\ell^i_{k+1} \in \cM_i$ and $i$ such that $s^i_{k+1} \in N_a(s^c_k)$\\
                Let $D^c:= \ell^c_k - U(s^c_k, a^c)$  \label{alg:mksl}
        }
        
		$(\overline \Bell^i_{k+1})_{i \in \cN }  \leftarrow$ Solve  \texttt{MMcMinKS}$[\{ (w_{i,j}, v^c_{i,j})_{i\in \cN,j \in  \cM_i}, D^c\}_{s^c_k \in J_k}]$
		$\textsc{Er}_{a^c}(s^c_k, \overline \Bell_{k+1}) := r(s^c_k)+(1-r(s^c_k)) \sum_{s^i_{k+1} \in N_{a^c}(s^c_k)} T(s^c_k, a^c, s^i_{k+1}) \cdot \texttt{DP}_{\textsc{Er}}(s^i_{k+1},\overline \ell^i_{k+1})$  for $s^c_k \in J_k$\\
		\If{$\sum_{s^c_k \in J_k}\textsc{Er}_{a^c}(s^c_k, \overline \Bell_{k+1}) <  \sum_{s^c_k \in J_k}\textsc{Er}(s^c_k)$}{ 
           \For{$s^c_k \in J_k$}{
			$\textsc{Er}(s^c_k) \leftarrow \textsc{Er}_{a^c}(s^c_k, \overline \Bell_{k+1})$\\
			$\textsc{Act}^c \leftarrow a^c$ \\
            $\textsc{Alloc}^c \leftarrow  (\overline\Bell^i_{k+1})_{s^i_{k+1} \in N_{a^c}(s^c_k)}$
	        }
          }
	      
	}
	\Return $(\textsc{Er}(s^c_k), \textsc{Act}^c, \textsc{Alloc}^c)_{s^c_k \in J_k}$
\end{algorithm}

To tackle (C)C-MDP with local transitions, we perform a tree decomposition: a transformation of the MDP graph into a tree where each node in the tree consists of a set of states. The tree nodes define a family of disjoint sets of states  $\cJ_k \subseteq 2^{\cS_k}$ at each time $k$  
 (level in the MDP graph) as 
\begin{align}
\cJ_k &:= \{J_k \subseteq \cS_k \mid \textsc{Reach}(s_k) \cap \textsc{Reach}(s'_k) \neq \varnothing,\notag\\
&\qquad \text{ for any pair } s_k, s'_k \in J_k \}. \label{eq:reach}
\end{align}

According to the local transition assumption,  $|J_k|$ is at most a constant $\psi$, a property that is necessary to obtain  a polynomial-time algorithm. 
Next, we show how to convert the allocation subproblem into a {\em multi-dimensional} version of McMinKS (denoted as MMcMinKS).
MMcMinKS extends definition~\ref{def:ks} allowing the value of each item to be a $|J_k|$-dimensional vector. 
The goal is to compute the total minimum weight such that the total value for each dimension is at least $D^c \in \RR_+$ for $c= \{1,...,|J_k|\}$. More formally, the problem can be defined as an ILP by replacing Cons.~\raf{con} by 
$\sum_{i \in \cN} \sum_{j\in \cM_i} v^c_{i,j} x_{i,j} \ge D^c,$  for $c=1,...,|J_k|$.

The key idea, presented in Alg.~\ref{alg2} (denoted by {\tt DynMDP}), is to operate on clusters of states $J_k$ instead of individual states as in {\tt dis-DynMDP}.  
{\tt mKS-Update} (Alg.~\ref{alg:update3}) provides a reduction from the allocation problem into MMcMinKS. As we have a tree structure, the problem structure remains similar to {\tt dis-DynMDP} except that here we require to solve an instance of multi-dimensional McMinKS. This can be done by a slight modification of {\tt Dyn-MinKS}. The basic idea is to round off the set of possible values to obtain a range, by which we can optimize over in polynomial time using dynamic programming. Thus, we create $|J_k|$ dimensional dynamic programming table {\tt TB}$(j, \rho^1,..., \rho^{|J_k|})$. Since $|J_k|\le \psi$  is a constant, the size of the table is polynomial in the input size.

 \begin{theorem} \label{thm:fptas3}
	Algorithm {\tt DynMDP} is an FPTAS for (C)C-MDP under local transition assumption.
    \vspace{-5pt}
\end{theorem}
A proof of the theorem can be obtained using that of Corollary~\ref{thm:fptas2} with a slight modification of Lemma~\ref{lem:minks} to account for higher dimensional dynamic programming table.

\vspace{-5pt}
\section{Conclusion}
This work provides the first fully polynomial-time approximation scheme for a class of constrained MDP under local transition. Since the problem is NP-Hard, our algorithm is the best polynomial-time approximation algorithm attainable in theory. We believe our results provide fundamental insights into the problem and can lead to the future development of algorithms and faster heuristics for (C)C-MDP and constrained reinforcement learning.

\bibliographystyle{abbrv}
\bibliography{reference}

\end{document}